%% file: main.tex
\documentclass[letterpaper, 10 pt, conference]{ieeeconf}  

\usepackage[utf8]{inputenc} 
\usepackage[T1]{fontenc}    
\usepackage{hyperref}       
\usepackage{url}            
\usepackage{booktabs}       
\usepackage{amsfonts}       
\usepackage{nicefrac}       
\usepackage{microtype}      
\usepackage[bottom]{footmisc}
\usepackage{graphicx}
\usepackage{multirow}
\usepackage{amsfonts}
\usepackage{adjustbox}
\usepackage[dvipsnames]{xcolor}

\usepackage{listings}
\usepackage{algorithm}
\usepackage{algpseudocode}

\usepackage{wrapfig}
\usepackage{caption}
\usepackage{subcaption,lipsum}
\usepackage{pifont}
\usepackage{bbm}
\usepackage{multicol}
\usepackage{afterpage}
\usepackage{placeins}
\usepackage{dblfloatfix}
\usepackage[flushleft]{threeparttable} %

\let\labelindent\relax
\usepackage{enumitem}

\usepackage{amsmath}
\usepackage{amssymb}
\usepackage{mathtools}

\usepackage{amsthm}
\theoremstyle{plain}
\newtheorem{theorem}{Theorem}[section]
\newtheorem{proposition}[theorem]{Proposition}

\DeclarePairedDelimiterX{\infdivx}[2]{(}{)}{%
  #1\;\delimsize\|\;#2%
}

\IEEEoverridecommandlockouts                              

\title{\LARGE \bf
FW-NKF: Frequency-Weighted Neural Kalman Filters
}

\author{%
Adnan Harun Dogan$^{*}$,
Berken Utku Demirel,
Christian Holz
\thanks{$\ast$ Corresponding author. 
All authors are with the Department of Computer Science, ETH Zürich, Switzerland$^{1}$.
Code is available at \url{https://github.com/eth-siplab/Frequency-weighted-neural-Kalman-filters}
}
}

\begin{document}

\maketitle
\thispagestyle{empty}
\pagestyle{empty}

\input{Sections/0-Abstract.tex}
\input{Sections/1-Introduction.tex}

\input{Sections/7-Related_Work}

\input{Sections/2_1-Prelim}
\input{Sections/2-Proposed_Method.tex}

\input{Sections/3-Experimental_Setup.tex}
\input{Sections/4-Results.tex}
\input{Sections/6-Conclusion.tex}




\bibliographystyle{IEEEtran}
\bibliography{IEEEabrv,sample}

\clearpage

\end{document}

%% file: Sections/0-Abstract.tex
\begin{abstract}
Robust state estimation is central to robotic autonomy, yet classical Kalman filters struggle with frequency-dependent disturbances and model mismatch such as sensor vibrations, electromagnetic interference, and periodic noise.
Although Deep Kalman Filter (DKF) variants extend the Extended Kalman Filtering (EKF) framework by learning latent transitions, they lack explicit mechanisms to suppress band-limited noise components that typically corrupt sensor measurements in real-world scenarios.
We introduce the Frequency-Weighted Neural Kalman Filter (FW-NKF), a unified hybrid approach that embeds a causal spectral-shaping operator into the Kalman measurement residual and jointly learns observation, and transition networks.
By adapting both the filter spectrum and the latent state representation, FW-NKF attenuates the noise-dominated frequency bands while capturing complex residual structures.
We conduct extensive experiments on four heterogeneous benchmarks, including chaotic systems such as multi-dimensional Lorenz systems and full-body inertial pose estimation, and find a reduction in localization error of up to 10\% as well as marked improvements in orientation accuracy.
Our ablation studies confirm that frequency weighting and deep latent-state modeling contribute to overall performance.
\end{abstract}

%% file: Sections/1-Introduction.tex
\section{Introduction}
Real-world sensor measurements, including drone inertial measurement units (IMUs), wheel encoders, and wearable IMUs, are often corrupted by band-specific disturbances such as high-frequency vibration jitter, low-frequency bias drift, and periodic electromagnetic interference \cite{haarnoja2018composable}.
Downstream tasks such as trajectory planning, flight stabilization, and avatar animation in VR/AR demand state estimates that are both accurate and temporally coherent~\cite{yi2021transpose,gong2019spiralnet}, yet even moderate spectral contamination can degrade quality~\cite{rempe2021humor}.

State estimation often relies on noisy IMU data, where integration drift causes small biases in acceleration or angular rate to accumulate over time.
These small biases can quickly become large position or orientation errors, especially when external corrections such as GPS or visual markers are unavailable~\cite{chen2018ionet}.
For example, in ground vehicles, this drift can result in missed turns or navigation failures.
For micro-drones operating with tight agility constraints, even minor pose errors can destabilize flight, preventing successful obstacle avoidance or trajectory tracking.
In human motion capture, drift in IMU-based pose estimation causes joint angle inconsistencies and body posture deviations, degrading the realism and accuracy of avatar representations in virtual reality~\cite{zhou2022toch}.

The classical Kalman filter attains minimum mean-squared error (MMSE) optimality only when process and measurement noise are white, Gaussian, and temporally uncorrelated~\cite{Krishnan2015,krishnan2017structured,jiang2017learning}—conditions famously called ``fairy tales for undergraduates''~\cite{thomson_jackknife, maybeck1982stochastic}.
When noise is instead colored, periodic, or band-limited~\cite{hazan2018spectral,oconnor_vibration,man_made_noise}, the filter's fixed covariance model can neither capture nor suppress these artifacts, causing estimation errors to accumulate across time steps~\cite{gu2016deep,kloss2021how}.
High-frequency IMU jitter~\cite{Sarkka2013} and low-frequency bias drift~\cite{corenflos2021differentiable} are two canonical examples that scalar $\boldsymbol{Q}$/$\boldsymbol{R}$ tuning cannot resolve.
EKF and UKF variants~\cite{wan2000unscented} relax the linearity requirement but inherit the same spectral blind spot, as their gain computation remains agnostic to the frequency content of the innovation.

In this work, we introduce the Frequency-Weighted Neural Kalman Filter (FW-NKF), a frequency-aware extension of Kalman filtering designed for state estimation.
Instead of relying on a fixed linear observation matrix, we use a learnable neural network that maps latent states to denoised sensor measurements, supervised by ground truth states.
This network implicitly learns to reconstruct noise-free sensor signals to enable a more accurate integration of the true state trajectory.
To introduce frequency-specific error, we apply a Fourier transform to the estimated noise-free and observed signals and minimize the mean squared error in the frequency domain to encourage the filter to suppress spectral components most responsible for drift and noise.
In addition, we introduce a learnable frequency-domain filter applied to the Kalman innovation term to selectively attenuate disruptive frequency bands during the correction step.
By integrating frequency awareness into both the observation model and update mechanism, FW-NKF achieves improved state estimation accuracy across diverse tasks while maintaining the computational efficiency compared to the non-linear filters.
Our contributions are:
\begin{itemize}[leftmargin=*]
    \item  A hybrid filtering framework that embeds a causal, frequency-weighted operator within a Kalman filter to learn and attenuate system-specific noise concentrated in particular frequency bands.
    \item A spectral-domain training objective that minimizes a frequency-weighted error between predicted and denoised sensor signals to suppress noise-dominated frequency components. This is motivated by the Cramér–Rao lower bound, since reducing effective measurement noise can lower the minimum achievable estimator variance under standard regularity assumptions.
    \item Extensive evaluation on four heterogeneous benchmarks---pendulum, Lorenz chaotic systems, a micro aerial vehicle (MAV) and human pose estimation---achieving up to 10\% reduction in state estimation while maintaining a competitive computational overhead.
\end{itemize}


%% file: Sections/7-Related_Work.tex
\section{Related Work}
\paragraph{Classical and Adaptive Kalman Filtering}
The Kalman filter~\cite{Kalman1960} remains the standard recursive estimator for linear Gaussian systems owing to its closed-form predict–update cycle and minimum-variance guarantees~\cite{Welch1995,Sarkka2013}.
When dynamics are nonlinear, the Extended Kalman Filter (EKF)~\cite{robust_extended_KF} linearizes the current estimate, while the Unscented Kalman Filter (UKF)~\cite{wan2000unscented} propagates deterministic sigma points to capture second-order statistics~\cite{Konatowski_Kaniewski_Matuszewski_2016,freirich2023perceptualkalmanfiltersonline}.
Both families, along with innovation-based covariance adaptation and Interacting Multiple Model (IMM) switching~\cite{BarShalom2001,Julier1997}, adjust gains to accommodate maneuvering targets or drifting noise statistics.
Nevertheless, all these variants treat measurement noise as spectrally flat, lacking a mechanism to selectively attenuate frequency-localized disturbances—the gap our work addresses.
KalmanNet~\cite{revach2022kalmannet} replaces the analytic Kalman gain with a learned GRU mapping, achieving competitive accuracy without explicit noise statistics.
Subsequent extensions add Bayesian weight posteriors for uncertainty calibration~\cite{dahan2023uncertainty} or recursive gain structures for long-horizon stability~\cite{mortada2024recursive}.
Variational state-space models~\cite{Krishnan2015,Fraccaro2017} further generalize by jointly learning nonlinear transitions and emissions, and particle-filter hybrids extend this to non-Gaussian posteriors~\cite{Sarkka2013}.
While these methods relax linearity and Gaussianity constraints, none incorporate spectral supervision or frequency-selective innovation filtering.
\paragraph{Frequency‐Weighted Filtering}
Frequency‐shaped Kalman filters enhance robustness by selectively emphasizing informative spectral bands.
By applying a causal frequency‐shaping operator to the innovation sequence, these methods minimize weighted estimation variances in targeted frequency ranges, effectively suppressing band‐limited noise~\cite{Einicke2006}.  
Designing suitable shaping functions requires prior knowledge of sensor noise or estimation of noise characteristics.
\paragraph{Application‐Specific Filtering and Localization}
Many real-world estimation problems, ranging from autonomous navigation to human motion tracking, rely on noisy sensory data prone to frequency-dependent disturbances~\cite{analog_filters_imu}.
Benchmarks UIP-DP pose tracking~\cite{Armani2024} collections demonstrate these challenges across diverse platforms, environments, and sensor modalities~\cite{Carlevaris-Bianco2016}.
While application-specific methods and task-oriented neural networks have been developed to address domain-specific dynamics~\cite{7152323}, they often overlook the structured nature of sensor noise, particularly its spectral characteristics.
Our frequency-aware filtering approach can complement these existing methods by explicitly attenuating noise-dominated frequency components, improving state estimation accuracy without handcrafted filters.


%% file: Sections/2_1-Prelim.tex
\section{Preliminaries}
\subsection{Notations}
We use lowercase letters (\( x \)) to denote scalar quantities, and bold lowercase letters (\( \boldsymbol{{x}} \)) to represent vectors, such as temporal sequences, while bold uppercase letters (\( \boldsymbol{X} \)) are used for matrices.
The discrete Fourier transform is denoted as $\mathcal{F}(.)$, yielding a complex variable as $\mathcal{F}_{\boldsymbol{x}}$.
\subsection{Linear Gaussian State-Space Models}
Our framework builds on the discrete-time linear Gaussian state-space model~\cite{Kalman1960,Welch1995}:
\begin{align}
\boldsymbol{x}_k &= \boldsymbol{F}_k\,\boldsymbol{x}_{k-1} + \boldsymbol{B}_k\,\boldsymbol{u}_k + \boldsymbol{w}_k,\quad &\boldsymbol{w}_k \sim \mathcal{N}(0,\boldsymbol{Q}_k), \label{eq:ssm_state}\\
\boldsymbol{y}_k &= \boldsymbol{H}_k\,\boldsymbol{x}_k + \boldsymbol{v}_k,\quad &\boldsymbol{v}_k \sim \mathcal{N}(0,\boldsymbol{R}_k), \label{eq:ssm_measure}
\end{align}
with state $\boldsymbol{x}_k{\in}\mathbb{R}^n$, control $\boldsymbol{u}_k{\in}\mathbb{R}^p$, and observation $\boldsymbol{y}_k{\in}\mathbb{R}^m$.
$\boldsymbol{F}_k{\in}\mathbb{R}^{n\times n}$ and $\boldsymbol{B}_k{\in}\mathbb{R}^{n\times p}$ govern state evolution, $\boldsymbol{Q}_k{\in}\mathbb{R}^{n\times n}$ models process uncertainty, $\boldsymbol{H}_k{\in}\mathbb{R}^{m\times n}$ projects states onto measurements, and $\boldsymbol{R}_k{\in}\mathbb{R}^{m\times m}$ encodes sensor noise~\cite{Sarkka2013}.
\subsection{Classical Kalman Filter}
Given measurements $\boldsymbol{y}_{1:t}$, the Kalman filter recursively computes the posterior $\boldsymbol{x}_t$ that minimizes the mean-squared estimation error under the model above~\cite{Kalman1960,Welch1995}:
\paragraph{Prediction Step}
\begin{align} \label{eq:kf_predict_x}
\hat{\boldsymbol{x}}_{t|t-1} = \boldsymbol{F} \hat{\boldsymbol{x}}_{t-1|t-1} + \boldsymbol{B} \boldsymbol{u}_t, \\ \boldsymbol{P}_{t|t-1} = \boldsymbol{F} \boldsymbol{P}_{t-1|t-1} \boldsymbol{F}^\top + \boldsymbol{Q}.
\end{align}
\paragraph{Update Step}
\begin{align}
\Delta \boldsymbol{y}_t &= \boldsymbol{y}_t - \boldsymbol{H}\hat{\boldsymbol{x}}_{t|t-1}, \label{eq:kf_innov}\\
\boldsymbol{K}_t &= \boldsymbol{P}_{t|t-1} \boldsymbol{H}^\top\bigl(\boldsymbol{H} \boldsymbol{P}_{t|t-1} \boldsymbol{H}^\top + \boldsymbol{R}\bigr)^{-1}, \label{eq:kf_gain}\\
\hat{\boldsymbol{x}}_{t|t} &= \hat{\boldsymbol{x}}_{t|t-1} + \boldsymbol{K}_t\,\Delta \boldsymbol{y}_t, \label{eq:kf_update_x}
\hspace{1mm}
\boldsymbol{P}_{t|t} = (\boldsymbol{I} - \boldsymbol{K}_t \boldsymbol{H})\,\boldsymbol{P}_{t|t-1}. 
\end{align}
The subscript $t|t{-}1$ indicates the prior (before incorporating $\boldsymbol{y}_t$) and $t|t$ the posterior.
Eqs.~\eqref{eq:kf_predict_x}--\eqref{eq:kf_update_x} alternate between propagating the state forward and correcting it via the innovation $\Delta\boldsymbol{y}_t$ scaled by the Kalman gain $\boldsymbol{K}_t$.
\subsection{Motivation}
Kalman filters assume stationary noise covariances \(Q\) and \(R\), which are suboptimal when sensor noise exhibits frequency‐dependent characteristics~\cite{Einicke2006}.  
In practice, sensor measurements contain structured disturbances such as low-frequency drift, high-frequency jitter, or periodic environmental noise, none of which are captured by the white noise assumption~\cite{Kalman1960}.
These spectral artifacts can lead to suboptimal state estimation, as the Kalman gain cannot distinguish between informative signal components and frequency-localized noise.
We propose frequency-weighted neural Kalman filter to introduce spectral awareness into both the observation model and the update step.
By combining an observation network trained with spectral-domain supervision and a causal filter applied to the innovation term, FW-NKF selectively suppresses noise-dominated frequencies while preserving task-relevant information.
This approach enables accurate state estimation in environments where classical filters fall short, particularly for sensor modalities like IMUs, where frequency-specific disturbances are common~\cite{oconnor_vibration}.


%% file: Sections/2-Proposed_Method.tex
\section{Method}
In this section, we present two key enhancements to the Kalman filter, which are a frequency-domain reconstruction loss for training (Algorithm~\ref{alg:spectral-loss}) and a learnable IIR innovation filtering mechanism (Algorithm~\ref{alg:iir}). 
These components address limitations in conventional Kalman filter implementations when dealing with systems that have distinct frequency characteristics with non-Gaussian noise.

\subsection{Incorporating Frequency Loss into the Kalman Filter}
\label{sec:frequency-loss}
The frequency-domain reconstruction loss introduces a critical improvement over traditional time-domain losses by comparing signals in the spectral domain rather than point-by-point in time.
Specifically, we minimize the difference between the magnitude spectra of predicted and ground-truth observations, obtained via the Fourier Transform.
This shifts the focus from waveform alignment to frequency content, which is often more relevant in systems affected by structured, non-Gaussian noise—such as IMU signals~\cite{analog_filters_imu, oconnor_vibration}.
We define the spectral loss during training as in Equation~\ref{eq:spectral_loss}.
\begin{align}\label{eq:spectral_loss}
\mathcal{L}_{\Phi} = \Big\lVert \lVert \mathcal{F}(\boldsymbol{\hat{y}}_t) \rVert - \lVert \mathcal{F}(\boldsymbol{\tilde{y}}_t) \rVert \Big\rVert_1
\end{align}
where $\boldsymbol{\hat{y}}_t = H(\boldsymbol{\hat{x}}_{t|t-1})$ is the predicted observation from the estimated state, and $\boldsymbol{\tilde{y}}_t = H(\boldsymbol{x}_t)$ is the reference noiseless observation derived from the ground truth. 
Although noiseless readings are not directly available, we approximate them by applying the observation model to the ground truth states during training.
Under the supervision, the observation model $H(\cdot)$ is trained to approximate noise-free observations, making $\boldsymbol{\tilde{y}}_t$ a valid target for guiding spectral reconstruction.

Unlike time-domain minimum square error, this loss is phase-invariant and robust to temporal misalignments.
The use of the $\ell_1$ norm further encourages sparsity in spectral differences, ensuring the model focuses on aligning dominant frequency components while ignoring minor discrepancies.
Algorithm~\ref{alg:spectral-loss} summarizes the implementation of incorporating the frequency loss to the original filter.
\begin{algorithm}
\caption{Integration of spectral loss into the Kalman filter training}
\begin{algorithmic}[1]
\Procedure{SpectralLoss}{$\{ \boldsymbol{x}_{t|t} \}_{t=1}^T$: filtered states, $\{ \boldsymbol{\hat{y}}_t \}_{t=1}^T$: predicted observations, $\{ \boldsymbol{x}_t \}_{t=1}^T$: true states}
    \State $\mathcal{L}_{\text{state}} \gets \frac{1}{T} \sum_{t=1}^{T} \| \hat{\boldsymbol{x}}_{t|t} - \boldsymbol{x}_t \|_2^2$ \Comment{State estimation loss}
    \State $\boldsymbol{\tilde{y}}_t \gets H(\boldsymbol{x}_t)$ \Comment{Estimated noiseless observation from ground truth state}
    \State $\mathcal{F}_{\boldsymbol{\tilde{y}}} \gets \lVert \mathcal{F} \{ \boldsymbol{\tilde{y}}_t \} \rVert$ \Comment{Magnitude spectrum of estimated noiseless observation}
    \State $\mathcal{F}_{\boldsymbol{\hat{y}}} \gets \lVert \mathcal{F} \{ \boldsymbol{\hat{y}}_t \} \rVert$ \Comment{Magnitude spectrum of predicted observation}
    \State $\mathcal{L}_{\Phi} \gets \frac{1}{\omega} \sum_{\omega=0}^{\pi/2} \lVert \mathcal{F}_{\boldsymbol{\hat{y}}} - \mathcal{F}_{\boldsymbol{\tilde{y}}} \rVert_1$ \Comment{Spectral-domain reconstruction loss}
    \State $\mathcal{L}_{\text{total}} \gets \mathcal{L}_{\text{state}} + \lambda_{\Phi} \cdot \mathcal{L}_{\Phi}$
    \State \Return $\mathcal{L}_{\text{total}}$
\EndProcedure
\end{algorithmic}
\label{alg:spectral-loss}
\end{algorithm}

The proposed frequency-domain loss extends traditional time-domain objectives by comparing the magnitude spectra of predicted and noiseless observations via the Fourier transform. 
This emphasizes spectral structure rather than point-by-point waveform differences.
It is especially useful when frequency matters more than exact time alignment, such as in IMUs with vibration and interference, and in settings where noise is approximately non-Gaussian~\cite{oconnor_vibration,analog_filters_imu}.

\subsection{Learnable IIR Innovation Filtering for Kalman Estimation}
\label{sec:learnable-iir}
We apply a learnable causal IIR filter directly to the innovation
\(\Delta \boldsymbol{y}_t\) defined in (\ref{eq:kf_innov}), where
\(\Delta \boldsymbol{y}_t = \boldsymbol{y}_t - \boldsymbol{H}\,\hat{\boldsymbol{x}}_{t|t-1}\).
Unless stated otherwise, the same filter is shared across measurement channels and acts componentwise.

\paragraph{Z-transform convention and delay.}
For a discrete vector sequence \(\boldsymbol{u}_t \in \mathbb{R}^{m}\), its one-sided Z-transform is
\(\boldsymbol{U}(z) = \sum_{t=0}^{\infty} \boldsymbol{u}_t\,z^{-t}\), with complex indeterminate \(z\).
Here \(z^{-1}\) denotes one-sample delay; more generally \(z^{-i}\) is an \(i\)-step delay.
When the same scalar IIR is applied to each channel, this is equivalent to a diagonal transfer
\(\boldsymbol{\Phi}(z) = \Phi(z)\,\boldsymbol{I}_m\).

\paragraph{Filter parameterization.}
We parameterize a causal IIR with learnable coefficients \(\{b_i\}_{i=0}^{p}\) and \(\{a_j\}_{j=1}^{q}\):
\begin{align}
\Phi(z) \;=\; \frac{B(z^{-1})}{A(z^{-1})}
\;=\; \frac{\sum_{i=0}^{p} b_i\,z^{-i}}{\,1+\sum_{j=1}^{q} a_j\,z^{-j}\,}\!.
\label{eq:iir-tf}
\end{align}

\paragraph{Time-domain realization and KF correction.}
Instead of the single-line ARMA recursion, we use the equivalent cascaded realization with an internal buffer \(\boldsymbol{s}_t \in \mathbb{R}^{m}\):
\begin{align}
\boldsymbol{s}_t \;&=\; -\sum_{j=1}^{q} a_j\,\boldsymbol{s}_{t-j} \;+\; \Delta \boldsymbol{y}_t,  \label{eq:pole_recursion} \\
\widetilde{\Delta \boldsymbol{y}}_t \;&=\; \sum_{i=0}^{p} b_i\,\boldsymbol{s}_{t-i}.  \label{eq:fir}
\end{align}
The filtered innovation \(\widetilde{\Delta \boldsymbol{y}}_t = \Phi \ast \Delta \boldsymbol{y}_t\) replaces the raw innovation in the Kalman update as in Equation~\ref{eq:kf_update_x}.
\begin{align}
\hat{\boldsymbol{x}}_{t|t} \;=\; \hat{\boldsymbol{x}}_{t|t-1} + \boldsymbol{K}_t\,\widetilde{\Delta \boldsymbol{y}}_t,
\label{eq:joseph}
\end{align}
Stability is enforced by constraining the poles of \(A(z^{-1})\) to lie strictly inside the unit circle (e.g., via a stable-root reparameterization of \(\{a_j\}\)); this preserves a well-posed recursion and, with the Joseph form, maintains positive semidefiniteness of \(\boldsymbol{P}_{t|t}\) under finite precision.

\begin{algorithm}
\caption{Adaptive Frequency-Weighted Innovation}
\begin{algorithmic}[1]
\Procedure{IIR-Innovation}{$\boldsymbol{y}_t$, $\hat{\boldsymbol{x}}_{t|t-1}$, coeffs. $\{b_i\}_{i=0}^{p}$, $\{a_j\}_{j=1}^{q}$}
    \State \textbf{Inputs:} observation \(\boldsymbol{y}_t \in \mathbb{R}^{m}\); predicted state \(\hat{\boldsymbol{x}}_{t|t-1} \in \mathbb{R}^{n}\); observation matrix \(\boldsymbol{H} \in \mathbb{R}^{m \times n}\).
    \State \textbf{Coeffs:} feedforward \(\{b_i\}_{i=0}^{p}\), feedback \(\{a_j\}_{j=1}^{q}\).
    \State \textbf{Persistent internal state:} lagged IIR states \(\{\boldsymbol{s}_{t-j}\}_{j=1}^{q}\) (initialized to \(\boldsymbol{0}\) at start).
    \State \(\Delta \boldsymbol{y}_t \gets \boldsymbol{y}_t -  \boldsymbol{H}\,\hat{\boldsymbol{x}}_{t|t-1}\) \Comment{Innovation from (\ref{eq:kf_innov})}
    \State \(\boldsymbol{s}_t \gets -\sum_{j=1}^{q} a_j\,\boldsymbol{s}_{t-j} + \Delta \boldsymbol{y}_t\) \Comment{Pole recursion, (\ref{eq:pole_recursion})}
    \State \(\widetilde{\Delta \boldsymbol{y}}_t \gets \sum_{i=0}^{p} b_i\,\boldsymbol{s}_{t-i}\) \Comment{FIR stage (\ref{eq:fir})}
    \State \textbf{State update:} shift and store \(\boldsymbol{s}_t\) as the newest lag (discard the oldest).
    \State \textbf{Output:} \(\widetilde{\Delta \boldsymbol{y}}_t\) \Comment{Use in KF: \(\hat{\boldsymbol{x}}_{t|t} = \hat{\boldsymbol{x}}_{t|t-1} + \boldsymbol{K}_t\,\widetilde{\Delta \boldsymbol{y}}_t\)}
    \State \textbf{Cost:} \(O((p+q)\,m)\) per step for componentwise filtering with a shared \(\Phi\).
\EndProcedure
\end{algorithmic}
\label{alg:iir}
\end{algorithm}
\subsection{End-to-End Integration in FW\textminus NKF}
The two components enter the forward pass at complementary points without altering the Kalman structure. During \emph{training}, the observation pathway is supervised in the spectral domain via Algorithm~\ref{alg:spectral-loss}, which compares \(\boldsymbol{\hat{y}}_t=H(\hat{\boldsymbol{x}}_{t|t-1})\) to the noiseless target \(\boldsymbol{\tilde{y}}_t=H(\boldsymbol{x}_t)\) through the loss in (\ref{eq:spectral_loss}). During \emph{inference and training}, the correction step replaces the raw innovation \(\Delta \boldsymbol{y}_t\) from (\ref{eq:kf_innov}) with the frequency-weighted innovation produced by the learnable IIR in (\ref{eq:pole_recursion})–(\ref{eq:fir}), equivalently \(\widetilde{\Delta \boldsymbol{y}}_t = [\Phi \ast \Delta \boldsymbol{y}]_t\). The time-varying dynamics and noise terms \((\boldsymbol{A}_t,\boldsymbol{b}_t,\boldsymbol{Q}_t,\boldsymbol{R}_t)\) are generated by a GRU-conditioned parameterization, while the IIR coefficients \(\{a_j\},\{b_i\}\) are learned with the stability constraint on \(A(z^{-1})\) as specified in Section~\ref{sec:learnable-iir}. When measurements are missing or unreliable, the optional mask \(\boldsymbol{m}_t\) gates the update.

\begin{algorithm}
\caption{FW-NKF Prediction \& Update (forward step)}
\begin{algorithmic}[1]
\Procedure{Forward Step}{$\boldsymbol{y}_t$: observation, $\boldsymbol{m}_t$: mask (optional), $\boldsymbol{u}_t$: control (optional)}
    \State \emph{Given cached} $\hat{\boldsymbol{x}}_{t-1}$, $\boldsymbol{P}_{t-1}$, encoder hidden $\boldsymbol{h}_{t-1}$, IIR lags $\{\boldsymbol{s}_{t-j}\}_{j=1}^{q}$, fixed $\boldsymbol{H}$
    \State $\boldsymbol{z}_t \gets [\hat{\boldsymbol{x}}_{t-1}, \boldsymbol{y}_t]$ \Comment{Append $\boldsymbol{u}_t$ if present}
    \State $(\boldsymbol{h}_t) \gets \mathrm{GRU}(\boldsymbol{z}_t, \boldsymbol{h}_{t-1})$
    \State $\boldsymbol{A}_t \gets \boldsymbol{I} + \Delta\boldsymbol{A}_t(\boldsymbol{h}_t)$ \Comment{Low-rank or diagonal gate}
    \State $\boldsymbol{b}_t \gets f_b(\boldsymbol{h}_t)$; \quad $\boldsymbol{Q}_t \gets f_Q(\boldsymbol{h}_t)$; \quad $\boldsymbol{R}_t \gets f_R(\boldsymbol{h}_t)$
    \State $\hat{\boldsymbol{x}}_{t|t-1} \gets \boldsymbol{A}_t \hat{\boldsymbol{x}}_{t-1} + \boldsymbol{b}_t$
    \State $\boldsymbol{P}_{t|t-1} \gets \boldsymbol{A}_t \boldsymbol{P}_{t-1} \boldsymbol{A}_t^\top + \boldsymbol{Q}_t$
    \State $\boldsymbol{S}_t \gets \boldsymbol{H}\boldsymbol{P}_{t|t-1}\boldsymbol{H}^\top + \boldsymbol{R}_t$
    \State $\boldsymbol{K}_t \gets \boldsymbol{P}_{t|t-1}\boldsymbol{H}^\top \boldsymbol{S}_t^{-1}$ \Comment{Cholesky-inverse}
    \State $\Delta \boldsymbol{y}_t \gets \boldsymbol{y}_t - \boldsymbol{H}\hat{\boldsymbol{x}}_{t|t-1}$ \Comment{Innovation}
    \State $\boldsymbol{s}_t \gets -\sum_{j=1}^{q} a_j\,\boldsymbol{s}_{t-j} + \Delta \boldsymbol{y}_t$ \Comment{IIR all-pole stage}
    \State $\widetilde{\Delta \boldsymbol{y}}_t \gets \sum_{i=0}^{p} b_i\,\boldsymbol{s}_{t-i}$ \Comment{IIR FIR stage}
    \State $\hat{\boldsymbol{x}}_{t|t}' \gets \hat{\boldsymbol{x}}_{t|t-1} + \boldsymbol{K}_t \widetilde{\Delta \boldsymbol{y}}_t$
    \State $\boldsymbol{P}_{t|t}' \gets (\boldsymbol{I} - \boldsymbol{K}_t \boldsymbol{H})\,\boldsymbol{P}_{t|t-1}\,(\boldsymbol{I} - \boldsymbol{K}_t \boldsymbol{H})^\top + \boldsymbol{K}_t \boldsymbol{R}_t \boldsymbol{K}_t^\top$ \Comment{Joseph form}
    \State $\hat{\boldsymbol{x}}_{t|t} \gets \boldsymbol{m}_t \hat{\boldsymbol{x}}_{t|t}' + (1-\boldsymbol{m}_t)\hat{\boldsymbol{x}}_{t|t-1}$ \Comment{Mask update}
    \State $\boldsymbol{P}_{t|t} \gets \boldsymbol{m}_t \boldsymbol{P}_{t|t}' + (1-\boldsymbol{m}_t)\boldsymbol{P}_{t|t-1}$
    \State update lags $\{\boldsymbol{s}_{t-j}\} \gets$ shift-in $\boldsymbol{s}_t$ and discard oldest \Comment{Roll IIR state}
    \State cache $\hat{\boldsymbol{x}}_{t} \gets \hat{\boldsymbol{x}}_{t|t}$, $\boldsymbol{P}_{t} \gets \boldsymbol{P}_{t|t}$, $\boldsymbol{h}_t$ \Comment{For the next step}
    \State \Return $\hat{\boldsymbol{x}}_{t|t}$, $\boldsymbol{P}_{t|t}$, $\boldsymbol{h}_t$, $\boldsymbol{s}_t$
\EndProcedure
\end{algorithmic}
\end{algorithm}




\begin{figure*}[t]
    \centering
    \includegraphics[width=0.8\linewidth]{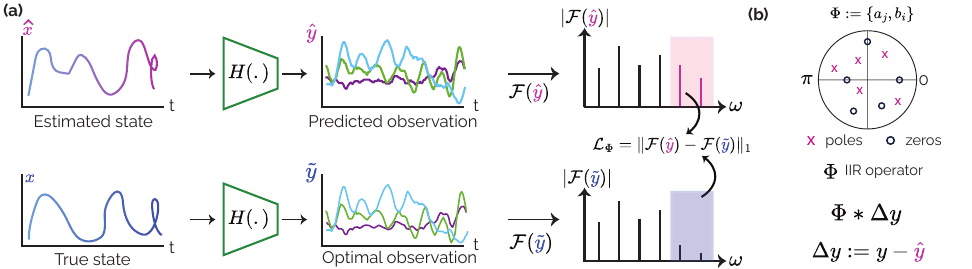}
    \caption{Overview of our method with spectral supervision.
    \textbf{(a)} The observation model $H(\cdot)$ predicts sensor measurements in two ways: one from the ground truth state $\boldsymbol{x}$ (supervised) and one from the estimated state $\boldsymbol{\tilde{x}}$. 
    This dual prediction enables to explicitly guide the observation model in reconstructing noise-free sensor signals, encouraging frequency-selective denoising.
    \textbf{(b)} The pole-zero plot of the learnable IIR filter $\Phi$, parameterized by coefficients $\{a_j, b_j\}$, which are optimized during training.
    This filter is applied causally to the innovation term (the difference between observed sensor data $\boldsymbol{y}$ and predicted observation $\boldsymbol{\hat{y}}$), selectively attenuating frequency components associated with measurement noise before the Kalman update.
}
    \label{fig:overall}
\end{figure*}
Proposition~\ref{prop1} formalizes a theoretical justification for our method.
It states that spectral filtering can improve state estimation when it suppresses noise in frequency bands that are weakly informative about the state, without removing state-dependent content.

In the measurement model $\boldsymbol{y}=h(\boldsymbol{x})+\boldsymbol{v}$, the noise $\boldsymbol{v}$ is often structured and frequency-dependent.
Applying a (causal) linear filter $\Phi$ yields the transformed observation $\boldsymbol{\tilde y}=\Phi\ast\boldsymbol{y}$, which both reshapes the measurement model and filters the noise.
When this transformation increases the information content of the observation about the state, the achievable estimation error decreases by the Cramér--Rao bound.

This result supports our approach of learning $\Phi$.
By training $\Phi$ to attenuate noise-dominated frequency components while preserving state-informative components, we improve the signal-to-noise ratio of the measurements, which is reflected in better state estimation in our experiments.
\begin{proposition}[Spectral Filtering]\label{prop1}
Let $\boldsymbol{y}=h(\boldsymbol{x})+\boldsymbol{v}$, where $\boldsymbol{x}\in\mathbb{R}^n$, $h$ is smooth, and $\boldsymbol{v}$ is zero-mean noise with covariance $\boldsymbol{R}=\mathbb{E}[\boldsymbol{v}\boldsymbol{v}^\top]$.
Let $\Phi$ be any (causal) linear filter and $\boldsymbol{\tilde y}=\Phi\,\boldsymbol{y}=\Phi\,h(\boldsymbol{x})+\tilde{\boldsymbol{v}}$ with $\tilde{\boldsymbol{v}}=\Phi\,\boldsymbol{v}$ and $\tilde{\boldsymbol{R}}=\Phi\,\boldsymbol{R}\,\Phi^\top$.
When the Fisher information in the filtered observation is no smaller than in the original observation, i.e., $I_{\tilde{\boldsymbol{y}}}(\boldsymbol{x})\succeq I_{\boldsymbol{y}}(\boldsymbol{x})$.
Then, by the Cramér--Rao bound, any unbiased estimator based on $\boldsymbol{\tilde y}$ has covariance no larger than one based on $\boldsymbol{y}$.
\end{proposition}

\begin{proof}
The Cramér--Rao bound states that for any unbiased estimator $\hat{\boldsymbol{x}}$ based on an observation $\boldsymbol{z}$, 
$\mathrm{Cov}(\hat{\boldsymbol{x}})\succeq I_{\boldsymbol{z}}(\boldsymbol{x})^{-1}$ (under standard regularity conditions). If $I_{\tilde{\boldsymbol{y}}}(\boldsymbol{x})\succeq I_{\boldsymbol{y}}(\boldsymbol{x})$, then $I_{\tilde{\boldsymbol{y}}}(\boldsymbol{x})^{-1}\preceq I_{\boldsymbol{y}}(\boldsymbol{x})^{-1}$, which yields the stated ordering on the covariance lower bounds. $\square$
\end{proof}


%% file: Sections/3-Experimental_Setup.tex
\section{Experiments}
We conducted experiments on four datasets spanning distinct applications to demonstrate the versatility of our method. Below, we provide a brief overview of each dataset.

\subsection{Datasets}
\subsubsection{Chaotic systems} 
\paragraph{Lorenz}
We simulate the Lorenz–63 system
$$
\nabla{\boldsymbol{x}}=\sigma(\boldsymbol{y}-\boldsymbol{x}),\quad
\nabla{\boldsymbol{y}}=\boldsymbol{x}(\rho-\boldsymbol{z})-\boldsymbol{y},\quad
\nabla{\boldsymbol{z}}=\boldsymbol{x}\boldsymbol{y}-\beta \boldsymbol{z}
$$
with the standard parameters $\sigma{=}10$, $\rho{=}28$, $\beta{=}8/3$.
Each trajectory has length $T{=}100$ steps with $\Delta t{=}0.01$\,s (total 1\,s). Initial states are drawn i.i.d.\ from $\mathcal N(0,5^2 I_3)$. We integrate the continuous dynamics using \texttt{odeint} and form latent states $\boldsymbol{x}_t\in\mathbb{R}^3$.
To stress the filter, we add small process perturbations to the \emph{simulated} path before sensing:
\begin{equation}
\boldsymbol{x}^{\text{proc}}_t=\boldsymbol{x}_t+\boldsymbol{\epsilon}_t, \qquad \boldsymbol{\epsilon}_t\sim\mathcal N(0,\,0.1^2 \boldsymbol{I}_3). 
\end{equation}
Observations are either partial or full. (i) \textbf{Partial:} $\boldsymbol{y}_t=[\boldsymbol{x}^{\text{proc}}_{t,1},\boldsymbol{x}^{\text{proc}}_{t,2}]^\top + \boldsymbol{\eta}_t\in\mathbb{R}^2$; (ii) \textbf{full:} $\boldsymbol{y}_t=\boldsymbol{x}^{\text{proc}}_t+\boldsymbol{\eta}_t\in\mathbb{R}^3$, with $\boldsymbol{\eta}_t\sim\mathcal N(0,\,0.5^2 \boldsymbol{I})$.
We report results for the partial setting (harder).
For linear-filter baselines, we use fixed discrete approximation of $F$ and $H$ (linearized around an equilibrium) as in Equation~\ref{eq:F_H}.
\begin{subequations}\label{eq:F_H}
\begin{align}
\boldsymbol{F} &= \boldsymbol{I}_3 + \Delta t\!
\begin{bmatrix}
-\sigma & \sigma & 0\\
\rho & -1 & 0\\
0 & 0 & -\beta
\end{bmatrix},\\
\boldsymbol{H}_{\text{partial}}&=\begin{bmatrix}1&0&0\\0&1&0\end{bmatrix},\quad
\boldsymbol{H}_{\text{full}}= \boldsymbol{I}_3.
\end{align}
\end{subequations}
We generate $N{=}1000$ trajectories and use an $80/10/10$ split (train/val/test) with fixed seeds for reproducibility.

\paragraph{Pendulum.}
We consider a damped, forced nonlinear pendulum with state $\boldsymbol{x}_t=[\theta_t,\omega_t]^\top$ and dynamics
\begin{equation}
\nabla\theta=\omega,\qquad
\nabla\omega=-(g/L)\sin\theta-(b/m)\,\omega + u_t/m,
\end{equation}
with $g{=}9.81$, $L{=}1.0$, $b{=}0.1$, $m{=}1.0$.
We discretize via forward Euler at $\Delta t{=}0.02$\,s for $T{=}100$ steps (2\,s). Controls are zero-mean Gaussian $u_t\sim \mathcal N(0,\;0.5^2)$. Process noise is injected additively after the Euler step: 
\begin{equation}
\boldsymbol{x}_{t+1}=\boldsymbol{x}_t+\Delta t\,f(\boldsymbol{x}_t,u_t)+\boldsymbol{\epsilon}_t, \qquad \boldsymbol{\epsilon}_t\sim\mathcal N(0,\,0.05^2 \boldsymbol{I}_2); 
\end{equation}
angles are wrapped to $(-\pi,\pi]$. Observations are Cartesian tip positions with Gaussian noise:
\begin{equation}
\boldsymbol{y}_t=\begin{bmatrix}L\sin\theta_t\\ -L\cos\theta_t\end{bmatrix}+\boldsymbol{\eta}_t,\qquad
\boldsymbol{y}_t\sim\mathcal N(0,\,0.1^2 \boldsymbol{I}_2).
\end{equation}
For methods~\cite{revach2022kalmannet},~\cite{dahan2023uncertainty} that require analytic solutions, we use the approximate analytic Jacobians implemented as,
\begin{equation}
\boldsymbol{F} = \begin{bmatrix}
1-\sigma\Delta t & \sigma\Delta t & 0\\
\rho\Delta t & 1-\Delta t & 0\\
0 & 0 & 1-\beta\Delta t
\end{bmatrix}
\end{equation}
\begin{equation}
\boldsymbol{H}_{\text{partial}}=\begin{bmatrix}1&0&0\\0&1&0\end{bmatrix},
\quad
\boldsymbol{H}_{\text{full}}=I_3.
\end{equation}
with $\bar\theta$ the linearization point. We generate $N{=}30000$ trajectories and use the same $80/10/10$ split.

\smallskip
\noindent\textit{Usage in our Kalman models.}
In both datasets the training target is the clean latent state $\boldsymbol{x}_t$ (pre-noise Lorenz states and post-process-noise pendulum states as returned by the generator). Measurement models and linearization matrices are provided to classical filters; learned Kalman variants receive the same observations and, where specified, learn adaptive $\boldsymbol{Q}_t$, $\boldsymbol{R}_t$ (and, for ablations, frequency-aware variants). All hyperparameters above exactly match the released defaults.

\subsubsection{Real-world}
\paragraph{UIP-DB}
UIP-DB contains 200 minutes of synchronized IMU and UWB range data from six wearable sensors alongside optical motion-capture ground truth across 25 motion types performed by 10 participants, designed for full-body inertial pose estimation with ranging constraints~\cite{Armani2024}.

The UIP-DB provides a challenging and realistic benchmark for Kalman filtering.
Full-body pose estimation from wearable IMUs is subject to drift due to sensor bias, while UWB ranging measurements are corrupted by high-frequency multipath effects.
Kalman filters are a natural fit for this setting, as they fuse heterogeneous measurements with noise statistics into a coherent state estimate. 

Importantly, the complementary characteristics of IMU and UWB measurements motivate frequency-aware modeling.
Low-frequency content often reflects slow pose evolution, while higher-frequency content captures rapid motion and noise.
A frequency-weighted Kalman approach adjusts the filter’s trust across bands, down-weighting noise-dominated components while preserving fast transitions.

\paragraph{EuRoC MAV}
The EuRoC MAV dataset~\cite{Burri2016} provides synchronized IMU measurements at 200\,Hz and motion-capture ground truth from a hexacopter flying in indoor environments.
We use 10 sequences (MH\_02--05, V1\_01--03, V2\_01--03) and formulate an IMU-based state estimation task where the observation is the 6D IMU (3-axis gyroscope $+$ 3-axis accelerometer), $\boldsymbol{y}_t\in\mathbb{R}^6$, and the state comprises position, velocity, and quaternion orientation, $\boldsymbol{x}_t\in\mathbb{R}^{10}$.
We segment each sequence into overlapping windows of length $T{=}200$ with stride 50, yielding 4{,}266 samples, and apply an $80/10/10$ train/val/test split.
For all Kalman filter variants, the observation model is a learned linear mapping $\boldsymbol{H}\in\mathbb{R}^{6\times 10}$.
\subsection{Evaluation Metrics}
We evaluate our method across diverse tasks, focusing on state estimation accuracy.
For all tasks, we report the state estimation error in $\ell_1$ and $\ell_2$ term.
The primary metric is the root-mean-square error (RMSE) between the estimated $\boldsymbol{\hat{x}}_t$ and the ground truth state $\boldsymbol{x}_t$, computed as in Equation~\ref{eq:metrics}.
\begin{equation}\label{eq:metrics}
\begin{aligned}
\text{MSE} &= \frac{1}{T}\sum_t \|\boldsymbol{e}_t\|_2^2, \qquad
\text{RMSE} = \sqrt{\text{MSE}}, \\
\text{MAE} &= \frac{1}{T}\sum_t \|\boldsymbol{e}_t\|_1, \qquad
\text{NRMSE} = \frac{\text{RMSE}}{\sqrt{\frac{1}{T}\sum_t \|\boldsymbol{x}_t\|_2^2}},
\end{aligned}
\end{equation}
where $\boldsymbol{e}_t \triangleq \hat{\boldsymbol{x}}_t - \boldsymbol{x}_t$.

In addition to per-frame errors, we also report the Average Trajectory Error (ATE), which measures the average Euclidean distance between the estimated trajectory and the ground truth trajectory over time. To capture worst-case deviations, we further report the maximum ATE, defined as the largest pointwise trajectory error observed during the sequence.

All methods are evaluated using the same ground truth references and synchronized time frames.
Standard deviations across random seeds are omitted from the tables when they fall below a significance threshold, indicating negligible variance.

\begin{figure}[t]
    \centering
    \includegraphics[width=\linewidth]{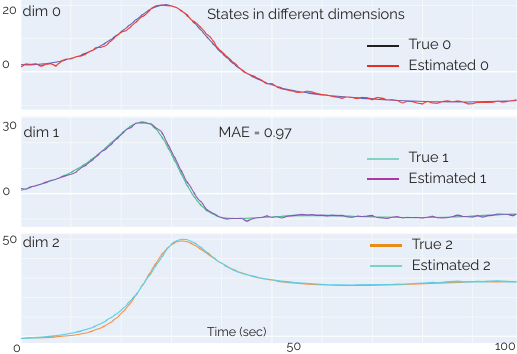}
    \caption{The state estimation performance of our method illustrated in 3-dimensional (x, y, z) Lorenz system where the estimations follow the true state in all dimensions with low mean absolute error. 
    All y-axes are dimensionless.
    }
    \label{fig:lorenz_ours}
\end{figure}

\begin{table}[t]
\caption{Performance comparison of different Kalman filtering methods on Pendulum dataset with varying weights.}
\centering
\begin{adjustbox}{max width=\columnwidth}
\label{tab:pendulum_results}
\renewcommand{\arraystretch}{1.1}
\begin{tabular}{@{}lllllll@{}}
\toprule
\multirow{2}{*}{Method} & \multirow{2}{*}{$\lambda_\Phi$} & \multicolumn{5}{c}{Performance Metrics} \\
\cmidrule(r{10pt}){3-7}
& & MAE $\downarrow$ & MSE $\downarrow$ & RMSE $\downarrow$ & NRMSE $\downarrow$ & R² $\uparrow$ \\
\midrule
\multicolumn{2}{l}{\hspace{-2mm}\textit{Deep Learning Methods}} & & & & & \\
FW-NKF (ours) & 0.0 & 0.75\footnotesize$\pm$0.08 & 1.24\footnotesize$\pm$0.28 & 1.10\footnotesize$\pm$0.12 & 0.51\footnotesize$\pm$0.05 & 0.74\footnotesize$\pm$0.05 \\
FW-NKF (ours) & 0.01 & \textbf{0.70\footnotesize$\pm$0.07} & \textbf{1.10\footnotesize$\pm$0.23} & \textbf{1.04\footnotesize$\pm$0.10} & \textbf{0.48\footnotesize$\pm$0.05} & \textbf{0.77\footnotesize$\pm$0.04} \\
FW-NKF (ours) & 0.1 & 0.99\footnotesize$\pm$0.14 & 2.03\footnotesize$\pm$0.63 & 1.42\footnotesize$\pm$0.21 & 0.65\footnotesize$\pm$0.09 & 0.59\footnotesize$\pm$0.12 \\
\midrule
BayesKNet~\cite{dahan2023uncertainty} & 0.0 & 1.19\footnotesize$\pm$0.06 & 2.93\footnotesize$\pm$0.32 & 1.71\footnotesize$\pm$0.09 & 0.79\footnotesize$\pm$0.04 & 0.40\footnotesize$\pm$0.06 \\
BayesKNet~\cite{dahan2023uncertainty} & 0.01 & 1.37\footnotesize$\pm$0.07 & 3.88\footnotesize$\pm$0.40 & 1.97\footnotesize$\pm$0.10 & 0.91\footnotesize$\pm$0.04 & 0.21\footnotesize$\pm$0.08 \\
BayesKNet~\cite{dahan2023uncertainty} & 0.1 & 1.38\footnotesize$\pm$0.07 & 3.96\footnotesize$\pm$0.41 & 1.99\footnotesize$\pm$0.10 & 0.92\footnotesize$\pm$0.04 & 0.20\footnotesize$\pm$0.08 \\
\midrule
Recursive KNet~\cite{mortada2024recursive} & 0.0 & 1.03\footnotesize$\pm$0.08 & 2.22\footnotesize$\pm$0.38 & 1.48\footnotesize$\pm$0.12 & 0.68\footnotesize$\pm$0.05 & 0.55\footnotesize$\pm$0.07 \\
Recursive KNet~\cite{mortada2024recursive} & 0.01 & \underline{0.98\footnotesize$\pm$0.08} & \underline{2.00\footnotesize$\pm$0.35} & \underline{1.41\footnotesize$\pm$0.12} & \underline{0.65\footnotesize$\pm$0.05} & \underline{0.59\footnotesize$\pm$0.07} \\
Recursive KNet~\cite{mortada2024recursive} & 0.1 & 1.10\footnotesize$\pm$0.10 & 2.52\footnotesize$\pm$0.48 & 1.58\footnotesize$\pm$0.15 & 0.73\footnotesize$\pm$0.06 & 0.49\footnotesize$\pm$0.09 \\
\midrule
Kalman Net~\cite{revach2022kalmannet} & 0.0 & 1.22\footnotesize$\pm$0.06 & 3.08\footnotesize$\pm$0.34 & 1.75\footnotesize$\pm$0.09 & 0.81\footnotesize$\pm$0.04 & 0.37\footnotesize$\pm$0.06 \\
Kalman Net~\cite{revach2022kalmannet} & 0.01 & 1.26\footnotesize$\pm$0.07 & 3.33\footnotesize$\pm$0.39 & 1.82\footnotesize$\pm$0.10 & 0.84\footnotesize$\pm$0.05 & 0.32\footnotesize$\pm$0.07 \\
Kalman Net~\cite{revach2022kalmannet} & 0.1 & 1.30\footnotesize$\pm$0.07 & 3.50\footnotesize$\pm$0.42 & 1.87\footnotesize$\pm$0.11 & 0.86\footnotesize$\pm$0.05 & 0.29\footnotesize$\pm$0.08 \\
\midrule
Recurrent KNet~\cite{becker2019recurrent} & 0.0 & 1.97\footnotesize$\pm$0.12 & 8.04\footnotesize$\pm$1.05 & 2.83\footnotesize$\pm$0.18 & 1.31\footnotesize$\pm$0.08 & -0.62\footnotesize$\pm$0.21 \\
Recurrent KNet~\cite{becker2019recurrent} & 0.01 & 2.03\footnotesize$\pm$0.13 & 8.54\footnotesize$\pm$1.12 & 2.92\footnotesize$\pm$0.19 & 1.35\footnotesize$\pm$0.08 & -0.72\footnotesize$\pm$0.22 \\
Recurrent KNet~\cite{becker2019recurrent} & 0.1 & 2.10\footnotesize$\pm$0.14 & 9.15\footnotesize$\pm$1.21 & 3.02\footnotesize$\pm$0.20 & 1.40\footnotesize$\pm$0.09 & -0.84\footnotesize$\pm$0.24 \\
\midrule
\multicolumn{2}{l}{\hspace{-2mm}\textit{Classical Methods}} & & & & & \\
Classical KF~\cite{Kalman1960} & 0.0 & 1.67\footnotesize$\pm$0.10 & 5.80\footnotesize$\pm$0.73 & 2.40\footnotesize$\pm$0.15 & 1.11\footnotesize$\pm$0.07 & -0.16\footnotesize$\pm$0.14 \\
Classical KF~\cite{Kalman1960} & 0.01 & 1.86\footnotesize$\pm$0.11 & 7.20\footnotesize$\pm$0.92 & 2.68\footnotesize$\pm$0.17 & 1.24\footnotesize$\pm$0.08 & -0.45\footnotesize$\pm$0.18 \\
Classical KF~\cite{Kalman1960} & 0.1 & 1.87\footnotesize$\pm$0.11 & 7.28\footnotesize$\pm$0.93 & 2.70\footnotesize$\pm$0.17 & 1.25\footnotesize$\pm$0.08 & -0.46\footnotesize$\pm$0.18 \\
\midrule
Autoreg KF~\cite{autoregressive_kalman} & 0.0 & 1.78\footnotesize$\pm$0.09 & 5.28\footnotesize$\pm$0.46 & 2.29\footnotesize$\pm$0.10 & 1.00\footnotesize$\pm$0.00 & 0.01\footnotesize$\pm$0.01 \\
Autoreg KF~\cite{autoregressive_kalman} & 0.01 & 1.87\footnotesize$\pm$0.10 & 5.81\footnotesize$\pm$0.51 & 2.41\footnotesize$\pm$0.10 & 1.05\footnotesize$\pm$0.01 & -0.17\footnotesize$\pm$0.02 \\
Autoreg KF~\cite{autoregressive_kalman} & 0.1 & 1.96\footnotesize$\pm$0.10 & 6.40\footnotesize$\pm$0.56 & 2.53\footnotesize$\pm$0.11 & 1.10\footnotesize$\pm$0.01 & -0.29\footnotesize$\pm$0.02 \\
\bottomrule
\multicolumn{2}{l}{Gain (\%)} & \textcolor{Green}{28.57\%} & \textcolor{Green}{45.00\%} & \textcolor{Green}{26.24\%} & \textcolor{Green}{26.15\%} & \textcolor{Green}{30.51\%} \\
\end{tabular}
\end{adjustbox}
\end{table}

\subsection{Baselines}
Our framework is designed to be modular and can be integrated with a wide range of Kalman filtering approaches.
To demonstrate its versatility, we also evaluate against six representative baselines that span classical, autoregressive, and neural Kalman filters.

\textbf{BayesKNet}~\cite{dahan2023uncertainty} places variational posteriors over the weight matrices of a gain-predicting network, yielding both a state estimate and an epistemic uncertainty at each step.

\textbf{Recursive KNet}~\cite{mortada2024recursive} feeds the previous Kalman gain back into the gain-prediction network at each time step to model gain dynamics and improving stability over long horizons.

\textbf{KalmanNet}~\cite{revach2022kalmannet} parameterize the Kalman gain as a GRU that ingests the innovation and prior covariance features, sidestepping the need for known $\boldsymbol{Q}$ and $\boldsymbol{R}$.

\textbf{Recurrent KNet}~\cite{becker2019recurrent} augments the gain network with an LSTM cell whose hidden state carries information from earlier time steps, capturing long-range temporal correlations.

\textbf{Classical KF}~\cite{Kalman1960} uses the analytic predict–update equations with fixed $\boldsymbol{F}$, $\boldsymbol{H}$, $\boldsymbol{Q}$, $\boldsymbol{R}$ matrices, serving as the MMSE-optimal reference under correct model assumptions.

\textbf{Autoregressive KF}~\cite{autoregressive_kalman} replaces the first-order Markov transition with a higher-order autoregressive process model, allowing richer temporal structure without neural components.

\begin{table}[t]
\caption{Performance comparison of different Kalman filtering methods on Lorenz dataset with varying weights.}
\centering
\begin{adjustbox}{max width=\columnwidth}
\label{tab:lorenz_results}
\renewcommand{\arraystretch}{1.1}
\begin{tabular}{@{}lllllll@{}}
\toprule
\multirow{2}{*}{Method} & \multirow{2}{*}{$\lambda_\Phi$} & \multicolumn{5}{c}{Performance Metrics} \\
\cmidrule(r{10pt}){3-7}
& & MAE $\downarrow$ & MSE $\downarrow$ & RMSE $\downarrow$ & NRMSE $\downarrow$ & R² $\uparrow$ \\
\midrule
\multicolumn{2}{l}{\hspace{-2mm}\textit{Deep Learning Methods}} & & & & & \\
FW-NKF (ours) & 0.0 & 0.62\footnotesize$\pm$0.10 & 0.81\footnotesize$\pm$0.26 & 0.87\footnotesize$\pm$0.14 & 0.05\footnotesize$\pm$0.09 & 0.97\footnotesize$\pm$0.01 \\
FW-NKF (ours) & 0.01 & \textbf{0.50\footnotesize$\pm$0.15} & \textbf{0.58\footnotesize$\pm$0.36} & \textbf{0.71\footnotesize$\pm$0.22} & \textbf{0.04\footnotesize$\pm$0.01} & \textbf{0.97\footnotesize$\pm$0.01} \\
FW-NKF (ours) & 0.1 & 8.18\footnotesize$\pm$1.04 & 119.82\footnotesize$\pm$30.51 & 10.94\footnotesize$\pm$1.38 & 0.66\footnotesize$\pm$0.08 & 0.54\footnotesize$\pm$0.11 \\
\midrule
BayesKNet~\cite{dahan2023uncertainty} & 0.0 & 4.49\footnotesize$\pm$1.03 & 35.84\footnotesize$\pm$16.56 & 5.94\footnotesize$\pm$1.36 & 0.36\footnotesize$\pm$0.08 & 0.86\footnotesize$\pm$0.06 \\
BayesKNet~\cite{dahan2023uncertainty} & 0.01 & 7.51\footnotesize$\pm$1.33 & 103.09\footnotesize$\pm$36.46 & 10.11\footnotesize$\pm$1.76 & 0.61\footnotesize$\pm$0.10 & 0.61\footnotesize$\pm$0.13 \\
BayesKNet~\cite{dahan2023uncertainty} & 0.1 & 7.32\footnotesize$\pm$1.02 & 98.20\footnotesize$\pm$27.96 & 9.87\footnotesize$\pm$1.38 & 0.60\footnotesize$\pm$0.08 & 0.63\footnotesize$\pm$0.10 \\
\midrule
Recursive KNet~\cite{mortada2024recursive} & 0.0 & 5.17\footnotesize$\pm$0.34 & 49.49\footnotesize$\pm$6.65 & 7.02\footnotesize$\pm$0.47 & 0.42\footnotesize$\pm$0.02 & 0.81\footnotesize$\pm$0.02 \\
Recursive KNet~\cite{mortada2024recursive} & 0.01 & \underline{4.24\footnotesize$\pm$0.24} & \underline{33.33\footnotesize$\pm$3.77} & \underline{5.76\footnotesize$\pm$0.32} & \underline{0.35\footnotesize$\pm$0.02} & \underline{0.87\footnotesize$\pm$0.01} \\
Recursive KNet~\cite{mortada2024recursive} & 0.1 & 4.85\footnotesize$\pm$0.18 & 43.03\footnotesize$\pm$3.36 & 6.55\footnotesize$\pm$0.25 & 0.40\footnotesize$\pm$0.01 & 0.83\footnotesize$\pm$0.01 \\
\midrule
Kalman Net~\cite{revach2022kalmannet} & 0.0 & 5.14\footnotesize$\pm$0.24 & 48.70\footnotesize$\pm$4.64 & 6.97\footnotesize$\pm$0.33 & 0.42\footnotesize$\pm$0.02 & 0.81\footnotesize$\pm$0.02 \\
Kalman Net~\cite{revach2022kalmannet} & 0.01 & 6.11\footnotesize$\pm$0.23 & 68.67\footnotesize$\pm$5.16 & 8.28\footnotesize$\pm$0.31 & 0.50\footnotesize$\pm$0.01 & 0.74\footnotesize$\pm$0.02 \\
Kalman Net~\cite{revach2022kalmannet} & 0.1 & 6.65\footnotesize$\pm$0.24 & 81.33\footnotesize$\pm$6.04 & 9.01\footnotesize$\pm$0.33 & 0.55\footnotesize$\pm$0.02 & 0.69\footnotesize$\pm$0.02 \\
\midrule
Recurrent KNet~\cite{becker2019recurrent} & 0.0 & 8.61\footnotesize$\pm$0.43 & 136.08\footnotesize$\pm$13.55 & 11.66\footnotesize$\pm$0.58 & 0.71\footnotesize$\pm$0.03 & 0.48\footnotesize$\pm$0.05 \\
Recurrent KNet~\cite{becker2019recurrent} & 0.01 & 9.60\footnotesize$\pm$0.49 & 169.32\footnotesize$\pm$17.32 & 13.01\footnotesize$\pm$0.66 & 0.79\footnotesize$\pm$0.04 & 0.36\footnotesize$\pm$0.06 \\
Recurrent KNet~\cite{becker2019recurrent} & 0.1 & 13.58\footnotesize$\pm$0.60 & 338.60\footnotesize$\pm$29.87 & 18.40\footnotesize$\pm$0.81 & 1.12\footnotesize$\pm$0.05 & -0.27\footnotesize$\pm$0.11 \\
\midrule
\multicolumn{2}{l}{\hspace{-2mm}\textit{Classical Methods}} & & & & & \\
Classical KF~\cite{Kalman1960} & 0.0 & 8.43\footnotesize$\pm$0.41 & 130.61\footnotesize$\pm$12.87 & 11.43\footnotesize$\pm$0.56 & 0.69\footnotesize$\pm$0.03 & 0.50\footnotesize$\pm$0.05 \\
Classical KF~\cite{Kalman1960} & 0.01 & 8.77\footnotesize$\pm$0.42 & 141.47\footnotesize$\pm$13.70 & 11.89\footnotesize$\pm$0.57 & 0.72\footnotesize$\pm$0.03 & 0.46\footnotesize$\pm$0.05 \\
Classical KF~\cite{Kalman1960} & 0.1 & 10.11\footnotesize$\pm$0.50 & 187.82\footnotesize$\pm$18.83 & 13.70\footnotesize$\pm$0.68 & 0.83\footnotesize$\pm$0.04 & 0.29\footnotesize$\pm$0.07 \\
\midrule
Autoreg KF~\cite{autoregressive_kalman} & 0.0 & 13.20\footnotesize$\pm$1.49 & 311.83\footnotesize$\pm$27.20 & 17.64\footnotesize$\pm$0.77 & 1.07\footnotesize$\pm$0.05 & -0.16\footnotesize$\pm$0.10 \\
Autoreg KF~\cite{autoregressive_kalman} & 0.01 & 13.23\footnotesize$\pm$1.50 & 313.06\footnotesize$\pm$27.33 & 17.68\footnotesize$\pm$0.78 & 1.08\footnotesize$\pm$0.05 & -0.16\footnotesize$\pm$0.10 \\
Autoreg KF~\cite{autoregressive_kalman} & 0.1 & 13.29\footnotesize$\pm$1.51 & 315.52\footnotesize$\pm$27.60 & 17.76\footnotesize$\pm$0.78 & 1.08\footnotesize$\pm$0.05 & -0.17\footnotesize$\pm$0.10 \\
\bottomrule
\multicolumn{2}{l}{Gain (\%)} & \textcolor{Green}{88.21\%} & \textcolor{Green}{98.26\%} & \textcolor{Green}{87.67\%} & \textcolor{Green}{88.57\%} & \textcolor{Green}{11.49\%} \\
\end{tabular}
\end{adjustbox}
\end{table}

\smallskip
\noindent\textit{Application of spectral loss to baselines.}
The frequency-domain reconstruction loss $\mathcal{L}_{\Phi}$ (Eq.~\ref{eq:spectral_loss}) is a \emph{training objective}, not an architectural modification.
It operates by projecting estimated and ground-truth states through the observation model $\boldsymbol{H}$ and comparing their magnitude spectra.
This makes it applicable to any differentiable Kalman filter whose parameters are optimized via gradient descent, regardless of whether those parameters are neural network weights or classical system matrices.

Concretely, Classical~KF and Autoregressive~KF parameterize the transition matrix $\boldsymbol{F}$, observation matrix $\boldsymbol{H}$, and noise covariances $\boldsymbol{Q},\boldsymbol{R}$ as learnable tensors.
When $\lambda_{\Phi}>0$, the spectral loss provides an additional gradient signal that encourages these matrices to better preserve the frequency content of the true observations.
For neural baselines (KalmanNet, BayesKNet, Recursive~KNet, Recurrent~KNet), the same loss regularizes their parameterizations.
We vary $\lambda_{\Phi}$ across methods to ensure a consistent and fair comparison of how each strategy benefits from frequency-domain supervision.
\begin{table}[t]
\caption{Performance comparison of Kalman filtering methods for state estimation on UIP-DB using IMU and UWB range measurements.}
\centering
\begin{adjustbox}{max width=\columnwidth}
\label{tab:uip_db_results}
\renewcommand{\arraystretch}{1.1}
\begin{tabular}{@{}lllllll@{}}
\toprule
\multirow{2}{*}{Method} & \multirow{2}{*}{$\lambda_\Phi$} & \multicolumn{2}{c}{Acceleration} & \multicolumn{1}{c}{Rotation} & \multicolumn{2}{c}{Inter Sensor Distance} \\
\cmidrule(r{15pt}){3-4}
\cmidrule(r{15pt}){5-5}
\cmidrule(r{15pt}){6-7}
& & RMSE $\downarrow$ & MAE $\downarrow$ & $\rho$ $\downarrow$ & RMSE $\downarrow$ & MAE $\downarrow$ \\
\midrule
\multicolumn{2}{l}{\hspace{-2mm}\textit{Deep Learning Methods}} & & & & & \\
FW-NKF (ours) & 0.0 & 4.37\footnotesize$\pm$0.24 & 13.46\footnotesize$\pm$0.83 & 0.42\footnotesize$\pm$0.01 & 1.41\footnotesize$\pm$0.07 & 6.80\footnotesize$\pm$0.34 \\
FW-NKF (ours) & 0.01 & \textbf{4.15\footnotesize$\pm$0.20} & \textbf{12.78\footnotesize$\pm$0.66} & \textbf{0.40\footnotesize$\pm$0.01} & \textbf{1.34\footnotesize$\pm$0.09} & \textbf{6.45\footnotesize$\pm$0.42} \\
FW-NKF (ours) & 0.1 & 4.73\footnotesize$\pm$0.33 & 14.57\footnotesize$\pm$1.15 & 0.45\footnotesize$\pm$0.02 & 1.53\footnotesize$\pm$0.11 & 7.35\footnotesize$\pm$0.57 \\
FW-NKF (ours) & 1.0 & 5.90\footnotesize$\pm$0.42 & 18.18\footnotesize$\pm$1.39 & 0.56\footnotesize$\pm$0.03 & 1.90\footnotesize$\pm$0.14 & 9.16\footnotesize$\pm$0.72 \\
\midrule
Recursive KNet~\cite{mortada2024recursive} & 0.0 & 5.46\footnotesize$\pm$0.39 & 16.82\footnotesize$\pm$1.26 & 0.52\footnotesize$\pm$0.03 & 1.76\footnotesize$\pm$0.13 & 8.48\footnotesize$\pm$0.66 \\
Recursive KNet~\cite{mortada2024recursive} & 0.01 & \underline{5.18\footnotesize$\pm$0.31} & \underline{15.95\footnotesize$\pm$1.00} & \underline{0.49\footnotesize$\pm$0.02} & \underline{1.67\footnotesize$\pm$0.11} & \underline{8.04\footnotesize$\pm$0.50} \\
Recursive KNet~\cite{mortada2024recursive} & 0.1 & 5.90\footnotesize$\pm$0.42 & 18.18\footnotesize$\pm$1.39 & 0.56\footnotesize$\pm$0.03 & 1.90\footnotesize$\pm$0.14 & 9.16\footnotesize$\pm$0.72 \\
\midrule
BayesKNet~\cite{dahan2023uncertainty} & 0.0 & 6.28\footnotesize$\pm$0.45 & 19.33\footnotesize$\pm$1.45 & 0.60\footnotesize$\pm$0.04 & 2.02\footnotesize$\pm$0.15 & 9.75\footnotesize$\pm$0.76 \\
BayesKNet~\cite{dahan2023uncertainty} & 0.01 & 5.95\footnotesize$\pm$0.36 & 18.33\footnotesize$\pm$1.15 & 0.57\footnotesize$\pm$0.03 & 1.92\footnotesize$\pm$0.12 & 9.24\footnotesize$\pm$0.58 \\
BayesKNet~\cite{dahan2023uncertainty} & 0.1 & 6.78\footnotesize$\pm$0.49 & 20.89\footnotesize$\pm$1.56 & 0.65\footnotesize$\pm$0.04 & 2.19\footnotesize$\pm$0.16 & 10.53\footnotesize$\pm$0.83 \\
\midrule
Kalman Net~\cite{revach2022kalmannet} & 0.0 & 7.21\footnotesize$\pm$0.52 & 22.20\footnotesize$\pm$1.66 & 0.69\footnotesize$\pm$0.04 & 2.32\footnotesize$\pm$0.17 & 11.20\footnotesize$\pm$0.88 \\
Kalman Net~\cite{revach2022kalmannet} & 0.01 & 6.83\footnotesize$\pm$0.41 & 21.05\footnotesize$\pm$1.32 & 0.65\footnotesize$\pm$0.04 & 2.20\footnotesize$\pm$0.14 & 10.60\footnotesize$\pm$0.66 \\
Kalman Net~\cite{revach2022kalmannet} & 0.1 & 7.79\footnotesize$\pm$0.56 & 24.00\footnotesize$\pm$1.80 & 0.74\footnotesize$\pm$0.05 & 2.51\footnotesize$\pm$0.18 & 12.10\footnotesize$\pm$0.95 \\
\midrule
Recurrent KNet~\cite{becker2019recurrent} & 0.0 & 10.06\footnotesize$\pm$0.72 & 31.00\footnotesize$\pm$2.32 & 0.96\footnotesize$\pm$0.06 & 3.24\footnotesize$\pm$0.23 & 15.64\footnotesize$\pm$1.23 \\
Recurrent KNet~\cite{becker2019recurrent} & 0.01 & 9.53\footnotesize$\pm$0.57 & 29.36\footnotesize$\pm$1.84 & 0.91\footnotesize$\pm$0.05 & 3.07\footnotesize$\pm$0.20 & 14.81\footnotesize$\pm$0.93 \\
Recurrent KNet~\cite{becker2019recurrent} & 0.1 & 10.87\footnotesize$\pm$0.78 & 33.49\footnotesize$\pm$2.51 & 1.04\footnotesize$\pm$0.07 & 3.51\footnotesize$\pm$0.27 & 16.89\footnotesize$\pm$1.33 \\
\midrule
\multicolumn{2}{l}{\hspace{-2mm}\textit{Classical Methods}} & & & & & \\
Classical KF~\cite{Kalman1960} & 0.0 & 9.05\footnotesize$\pm$0.65 & 27.89\footnotesize$\pm$2.09 & 0.86\footnotesize$\pm$0.05 & 2.92\footnotesize$\pm$0.22 & 14.06\footnotesize$\pm$1.11 \\
Classical KF~\cite{Kalman1960} & 0.01 & 8.58\footnotesize$\pm$0.52 & 26.43\footnotesize$\pm$1.66 & 0.82\footnotesize$\pm$0.05 & 2.77\footnotesize$\pm$0.18 & 13.31\footnotesize$\pm$0.83 \\
Classical KF~\cite{Kalman1960} & 0.1 & 9.78\footnotesize$\pm$0.70 & 30.14\footnotesize$\pm$2.26 & 0.93\footnotesize$\pm$0.06 & 3.15\footnotesize$\pm$0.24 & 15.20\footnotesize$\pm$1.20 \\
\midrule
Autoreg KF~\cite{autoregressive_kalman} & 0.0 & 10.66\footnotesize$\pm$0.77 & 32.85\footnotesize$\pm$2.46 & 1.02\footnotesize$\pm$0.06 & 3.44\footnotesize$\pm$0.26 & 16.57\footnotesize$\pm$1.31 \\
Autoreg KF~\cite{autoregressive_kalman} & 0.01 & 10.10\footnotesize$\pm$0.61 & 31.12\footnotesize$\pm$1.95 & 0.96\footnotesize$\pm$0.06 & 3.26\footnotesize$\pm$0.21 & 15.69\footnotesize$\pm$0.98 \\
Autoreg KF~\cite{autoregressive_kalman} & 0.1 & 11.52\footnotesize$\pm$0.83 & 35.51\footnotesize$\pm$2.66 & 1.10\footnotesize$\pm$0.07 & 3.72\footnotesize$\pm$0.28 & 17.91\footnotesize$\pm$1.41 \\
\bottomrule
\multicolumn{2}{l}{Gain (\%)} & \textcolor{Green}{19.88\%} & \textcolor{Green}{19.87\%} & \textcolor{Green}{18.37\%} & \textcolor{Green}{19.76\%} & \textcolor{Green}{19.78\%} \\
\end{tabular}
\end{adjustbox}
\end{table}

%% file: Sections/4-Results.tex
\begin{table}[t]
\caption{Runtime and model size comparison. Inference measured on a single NVIDIA RTX 6000 (batch$=$1, seq$=$100, averaged over 200 runs after 50 warm-up passes).}
\centering
\label{tab:runtime}
\renewcommand{\arraystretch}{1.1}
\begin{tabular}{@{}lrr@{}}
\toprule
Method & \#Params & ms/step \\
\midrule
\multicolumn{1}{l}{\textit{Deep Learning Methods}} & & \\
FW-NKF (ours) & 607,023 & 0.686 \\
BayesKNet~\cite{dahan2023uncertainty} & 106,279 & 2.882 \\
Recursive KNet~\cite{mortada2024recursive} & 7,351 & 1.222 \\
Kalman Net~\cite{revach2022kalmannet} & 15,476 & 0.742 \\
Recurrent KNet~\cite{becker2019recurrent} & 25,421 & 0.570 \\
\midrule
\multicolumn{1}{l}{\textit{Classical Methods}} & & \\
Classical KF~\cite{Kalman1960} & 28 & 0.276 \\
Autoreg KF~\cite{autoregressive_kalman} & 61 & 0.272 \\
\bottomrule
\end{tabular}
\end{table}


\begin{table}[t]
\caption{Performance comparison of Kalman filtering methods on EuRoC MAV IMU-based state estimation.
Recursive KNet is excluded due to memory constraints ($>$47\,GB).}
\centering
\begin{adjustbox}{max width=\columnwidth}
\label{tab:euroc_results}
\renewcommand{\arraystretch}{1.1}
\begin{tabular}{@{}lllllll@{}}
\toprule
\multirow{2}{*}{Method} & \multirow{2}{*}{$\lambda_\Phi$} & \multicolumn{5}{c}{Performance Metrics} \\
\cmidrule(r{10pt}){3-7}
& & MAE $\downarrow$ & MSE $\downarrow$ & RMSE $\downarrow$ & NRMSE $\downarrow$ & R² $\uparrow$ \\
\midrule
\multicolumn{2}{l}{\hspace{-2mm}\textit{Deep Learning Methods}} & & & & & \\
FW-NKF (ours) & 0.0 & \textbf{0.12\footnotesize$\pm$0.00} & \textbf{0.03\footnotesize$\pm$0.00} & \textbf{0.19\footnotesize$\pm$0.00} & \textbf{0.11\footnotesize$\pm$0.01} & \textbf{0.99\footnotesize$\pm$0.00} \\
FW-NKF (ours) & 0.01 & 0.12\footnotesize$\pm$0.01 & 0.04\footnotesize$\pm$0.00 & 0.19\footnotesize$\pm$0.01 & 0.11\footnotesize$\pm$0.00 & 0.99\footnotesize$\pm$0.00 \\
FW-NKF (ours) & 0.1 & 0.13\footnotesize$\pm$0.01 & 0.04\footnotesize$\pm$0.00 & 0.20\footnotesize$\pm$0.01 & 0.11\footnotesize$\pm$0.00 & 0.99\footnotesize$\pm$0.00 \\
\midrule
BayesKNet~\cite{dahan2023uncertainty} & 0.0 & 0.75\footnotesize$\pm$0.04 & 2.36\footnotesize$\pm$0.36 & 1.53\footnotesize$\pm$0.11 & 0.86\footnotesize$\pm$0.01 & 0.26\footnotesize$\pm$0.02 \\
BayesKNet~\cite{dahan2023uncertainty} & 0.01 & 0.75\footnotesize$\pm$0.04 & 2.36\footnotesize$\pm$0.36 & 1.53\footnotesize$\pm$0.11 & 0.86\footnotesize$\pm$0.01 & 0.26\footnotesize$\pm$0.02 \\
BayesKNet~\cite{dahan2023uncertainty} & 0.1 & 0.75\footnotesize$\pm$0.04 & 2.36\footnotesize$\pm$0.36 & 1.53\footnotesize$\pm$0.11 & 0.86\footnotesize$\pm$0.01 & 0.26\footnotesize$\pm$0.02 \\
\midrule
Kalman Net~\cite{revach2022kalmannet} & 0.0 & \underline{0.68\footnotesize$\pm$0.03} & \underline{1.81\footnotesize$\pm$0.24} & \underline{1.34\footnotesize$\pm$0.09} & \underline{0.75\footnotesize$\pm$0.06} & \underline{0.43\footnotesize$\pm$0.09} \\
Kalman Net~\cite{revach2022kalmannet} & 0.01 & 0.68\footnotesize$\pm$0.03 & 1.81\footnotesize$\pm$0.24 & 1.34\footnotesize$\pm$0.09 & 0.75\footnotesize$\pm$0.06 & 0.43\footnotesize$\pm$0.09 \\
Kalman Net~\cite{revach2022kalmannet} & 0.1 & 0.68\footnotesize$\pm$0.03 & 1.81\footnotesize$\pm$0.24 & 1.34\footnotesize$\pm$0.09 & 0.75\footnotesize$\pm$0.06 & 0.43\footnotesize$\pm$0.09 \\
\midrule
Recurrent KNet~\cite{becker2019recurrent} & 0.0 & 0.73\footnotesize$\pm$0.02 & 2.26\footnotesize$\pm$0.29 & 1.50\footnotesize$\pm$0.09 & 0.85\footnotesize$\pm$0.01 & 0.28\footnotesize$\pm$0.02 \\
Recurrent KNet~\cite{becker2019recurrent} & 0.01 & 0.73\footnotesize$\pm$0.02 & 2.26\footnotesize$\pm$0.29 & 1.50\footnotesize$\pm$0.09 & 0.85\footnotesize$\pm$0.01 & 0.28\footnotesize$\pm$0.02 \\
Recurrent KNet~\cite{becker2019recurrent} & 0.1 & 0.73\footnotesize$\pm$0.02 & 2.26\footnotesize$\pm$0.29 & 1.50\footnotesize$\pm$0.09 & 0.85\footnotesize$\pm$0.01 & 0.28\footnotesize$\pm$0.02 \\
\midrule
\multicolumn{2}{l}{\hspace{-2mm}\textit{Classical Methods}} & & & & & \\
Classical KF~\cite{Kalman1960} & 0.0 & 0.82 & 2.42 & 1.56 & 0.92 & 0.15 \\
Classical KF~\cite{Kalman1960} & 0.01 & 0.85 & 2.13 & 1.46 & 0.87 & 0.25 \\
Classical KF~\cite{Kalman1960} & 0.1 & 1.35\footnotesize$\pm$0.70 & 6.72\footnotesize$\pm$5.46 & 2.30\footnotesize$\pm$1.19 & 1.23\footnotesize$\pm$0.58 & $-$0.86\footnotesize$\pm$1.43 \\
\bottomrule
\multicolumn{2}{l}{Gain (\%)} & \textcolor{Green}{82.35\%} & \textcolor{Green}{98.34\%} & \textcolor{Green}{85.82\%} & \textcolor{Green}{85.33\%} & \textcolor{Green}{130.23\%} \\
\end{tabular}
\end{adjustbox}
\end{table}

%% file: Sections/6-Conclusion.tex
\section{Discussion}

\subsection{Performance Analysis Across Benchmarks}

The results across Tables~\ref{tab:pendulum_results}--\ref{tab:uip_db_results} suggest that the benefits of FW-NKF depend on the spectral structure of the underlying estimation problem.
On the pendulum benchmark (Table~\ref{tab:pendulum_results}), FW-NKF achieved the lowest MAE ($0.708\pm0.08$) relative to BayesKNet ($1.193\pm0.065$), corresponding to a 40.7\% reduction in error.
This improvement is consistent with the role of the frequency-domain reconstruction loss, which reduces high-frequency fluctuations in the observed Cartesian tip position while preserving the lower-frequency dynamics that are relevant for state tracking.

The gains are even larger on the Lorenz system (Table~\ref{tab:lorenz_results}), where FW-NKF reduced MAE from $4.24\pm0.24$ for Recursive KNet to $0.50\pm0.15$.
One possible explanation is that chaotic systems are especially sensitive to measurement corruption, so suppressing noise-dominated spectral components can improve both observation quality and downstream state propagation.
More broadly, these results suggest that frequency-aware innovation modeling is particularly useful when measurement noise overlaps with nonlinear or unstable dynamics, where small observation errors can rapidly accumulate over time.

\subsection{Limitations of Baselines}
Classical KF performance degrades across datasets, consistent with violations of the white Gaussian noise assumption.
Neural Kalman baselines improve on the classical filter, but still underperform FW-NKF, suggesting that learned filtering alone is insufficient when noise is frequency-dependent.

BayesKNet is more variable across benchmarks, performing competitively on pendulum but degrading substantially on Lorenz.
This pattern suggests that variational uncertainty estimation alone does not provide enough structure to handle spectrally complex and unstable dynamics.

\subsection{Component Contribution Analysis}

\paragraph{Frequency-Domain Loss} The spectral reconstruction loss (Eq.~\ref{eq:spectral_loss}) demonstrates clear benefits across datasets.
In UIP-DB (Table~\ref{tab:uip_db_results}), FW-NKF achieves $4.15\pm0.20$ RMSE for acceleration estimation versus Recursive KNet's $5.18 \pm 0.31$—a 19.9\% improvement resulting from guiding the observation network $H(\cdot)$ toward noise-free IMU reconstruction.

\paragraph{IIR Innovation Filtering} The learnable IIR innovation filter improves noise suppression.
On UIP-DB, orientation MAE drops to $1.34 \pm 0.09$ compared to $2.77 \pm 0.18$ for the classical KF, consistent with the learned coefficients attenuating dominated noise bands before the Kalman update.

\paragraph{Frequency Weight Selection}
Performance is best at moderate frequency weights ($\lambda=0.01$--$0.05$), indicating that mild spectral guidance is sufficient.
Larger $\lambda$ values reduce performance, as seen on the pendulum benchmark where MAE increases from $0.708 \pm 0.075$ at $\lambda=0.01$ to $1.205 \pm 0.179$ at $\lambda=5.0$.
This suggests that excessive spectral emphasis can interfere with the temporal objectives of sequential state estimation.



\section{Conclusion}
We presented FW-NKF, a frequency-aware neural Kalman filtering framework that improves state estimation by shaping the innovation signal and supervising observation learning in the spectral domain. 
Across chaotic systems and real-world sensor-fusion benchmarks, the results show that explicitly modeling frequency-dependent noise can yield more accurate and consistent estimates than conventional Kalman variants.
This work highlights that state estimation should not treat all measurement error as spectrally uniform.
In real-world, disturbances often concentrate in specific bands, and exploiting that structure improves robustness in downstream tasks such as navigation, motion tracking, and control.
The fact that moderate spectral weighting also benefits multiple differentiable Kalman baselines suggests that frequency-aware objectives may provide a general direction for building filters that better match the statistics of real sensor data.
